\documentclass[10pt]{article}
\usepackage{fullpage}
\usepackage{times,amsmath,amsfonts,amssymb,url,color,graphicx}

\usepackage{algorithm}
\usepackage{algorithmic}
\algsetup{indent=2em}

\usepackage{etaremune}

\usepackage{tikz}

\newcommand{\reals}{\mathbb{R}}

\newtheorem{definition}{Definition}
\newtheorem{lemma}{Lemma}
\newtheorem{corollary}{Corollary}
\newtheorem{example}{Example}
\newcommand{\BlackBox}{\rule{1.5ex}{1.5ex}}  
\newenvironment{proof}{\par\noindent{\bf Proof\ }}{\hfill\BlackBox\\[2mm]}

\newcommand{\figref}[1]{Figure~\ref{#1}}

\DeclareMathOperator*{\prob}{\mathbb{P}}

\DeclareMathOperator*{\Inf}{Inf}

\newcommand{\todo}[1]{{\bf \color{red} TODO: #1 }}

\begin{document}

\title{On the Sample Complexity of End-to-end Training vs. Semantic Abstraction Training}
\date{Mobileye}

\author{Shai Shalev-Shwartz \and Amnon Shashua}

\maketitle

\begin{abstract} 
We compare the end-to-end training approach to a
  modular approach in which a system is decomposed into semantically
  meaningful components. We focus on the sample complexity aspect, in
  the regime where an extremely high accuracy is necessary, as is the
  case in autonomous driving applications. We demonstrate cases in
  which the number of training examples required by the end-to-end
  approach is exponentially larger than the number of examples
  required by the semantic abstraction approach. 
\end{abstract}

\section{Introduction}

The recent impressive empirical success of deep learning lead
researchers to attempt building complicated systems by an end-to-end
training procedure. That is, training examples, which are pairs of an
input to the system and the desired output of the system, are
generated, and a single artificial neural network is trained to mimic
this input-output relationship. An alternative approach is to first
break the system into sub modules, where each individual module has a
clear semantic meaning. Second, each individual module is constructed
either by using a machine learning approach (e.g. the module is
trained in an end-to-end manner) or by relying on manual engineering
The choice of which option to use for each module is based on
empirical success. We call this approach Semantic Abstraction.

To further demonstrate the two approaches, consider a simplified
autonomous driving system, where a car driving in a highway should
keep its lane and adapt its speed according to other vehicles. The
input to the system is the sensory input (e.g., a video stream from a
camera and a radar signal). The output is a two dimensional vector
consisting of a steering command and an acceleration/deceleration
command. The end-to-end approach will train a single deep network
whose input is the sensory input and whose output is the two
dimensional vector of control commands. The training examples are
pairs of input-output to the whole system. In contrast, the semantic
abstraction approach will break the system into several sub-modules.
E.g., one module should  detect
vehicles based on the camera, another one should detect vehicles based
on the radar, a third module should fuse the two sources of
information. Other module detect lanes, and other modules 
make high level driving policy decisions (e.g. ``follow the car in
front of you'' or ``be careful from the car on your right because it
is likely to cut into your lane''). Finally, a low level control
module provides the two dimensional vector of control commands. 
 
Both approaches are far from being new. For example, the
end-to-end approach to autonomous driving dates back to
\cite{pomerleau1989alvinn}. For a detailed description and references,
see \cite{chen2015deepdriving}. There, the end-to-end approach is
called ``behavior reflex'' and the semantic abstraction approach is
called ``mediated perception''. We note that the term ``semantic
abstraction'' is adapted from the ``temporal abstraction'' approach to
reinforcement learning \cite{sutton1999between}.

There are several advantages and disadvantages of the two approaches
(see for example
\cite{gibson2014ecological,ullman1980against,chen2015deepdriving}). In
this paper we focus on the amount of data required for the training
process and for validating the quality of the learnt system. On one
hand, the advantage of end-to-end training is that we do not need
supervision for individual sub-modules of the system. On the other
hand, as we formally show in the next section, in some situations the
overall number of examples required by the end-to-end approach might
be exponentially larger than the number of examples required by the
semantic abstraction approach.

\section{Main Results}

Consider the problem of learning a system $f$ that maps from a domain
$X$ into a domain $Y$. Let $\ell$ be a loss function that determines
failures, namely, $\ell(x,f(x)) = 1$ if $f$ fails on the input $x \in
X$ and $\ell(x,f(x)) = 0$ otherwise. 

Following the standard PAC learning model (see for example
\cite{MLbook} for a precise definition), we define the
$(\epsilon,\delta)$-\emph{sample complexity} of learning to be the
number of training examples required such that there exists a learning
algorithm that with probability of at least $1-\delta$ (over the
random choice of the training examples) outputs a system $f$ for
which $\prob[\ell(x,f(x)) = 1] \le \epsilon$. The probability is
with respect to a random choice of $x$ according to some (unknown)
distribution $D$ over the domain $X$. 

Assuming there is some perfect system (which fails with probability
$0$), classical VC theory (see again \cite{MLbook} for a reference) tells us
that the sample complexity is (ignoring constants and logarithmic
term) $\textrm{VC}(\ell \circ \mathcal{F})/\epsilon$, where
$\mathcal{F}$ is the set of systems we aim to learn, $\ell \circ
\mathcal{F} = \{x \mapsto \ell(x,f(x)) : f \in \mathcal{F}\}$, 
and ``VC'' is the VC dimension. 

One may argue that this is a worst-case bound (where the worst
situation is w.r.t. the underlying distribution $D$), and for ``real world''
distributions the number of examples can be smaller. To make a
stronger lower bound, suppose one already trained a system $f$, and
let us consider the simpler task of just distinguishing between the
two cases $\prob[\ell(x,f(x)) = 1] \le \epsilon$ or
$\prob[\ell(x,f(x)) = 1] \ge 2\epsilon$. In other words, we
are considering a ``validation'' task, determining if our system is
good enough or not. Observe that if $\prob[\ell(x,f(x)) = 1] =
2\epsilon$ then the expected number of examples we need to observe in
order to see a single failure is $1/(2\epsilon)$. Therefore, if the
number of examples is significantly smaller than $1/(2\epsilon)$ we
have no way to distinguish between the case $\prob[\ell(x,f(x)) = 1] =
2\epsilon$ and the case $\prob[\ell(x,f(x)) = 1] \le \epsilon$. It
follows that the sample complexity of this validation task is
$\Omega(1/\epsilon)$. 

In some applications, the required $\epsilon$ is extremely small. For
example, in the autonomous driving application, we expect our system
to run properly on many cars for many years, which yields an extremely
large sample complexity.  For an end-to-end system, it is unavoidable
to require a sample of size $\Omega(1/\epsilon)$ just for validating
the system, and the sample complexity of training such a system is
likely to be several orders of magnitude larger.

We next show that decomposing the problem into semantically meaningful
sub-modules may lead to a significantly smaller sample complexity. 
Let $g_1,\ldots,g_k$
be functions, where for every $i$, $g_i : X \to
\{0,1\}$ is a boolean function that indicates a failure of some sub-module of
our full system. For example, $g_1$ corresponds to a failure of a
sub-module that prevents accidents with other vehicles, $g_2$
corresponds to a failure of a sub-module that prevents hitting pedestrians, and
so on and so forth. 

Let us first focus on some individual sub-module $g_i$. To simplify
the presentation, we omit the under-script, and aim at bounding
$\prob[g]$, which is the probability of failure of this sub-module. As
mentioned before, since we would like to be in a situation that
$\prob[g]$ is extremely small (meaning that the performance of our
sub-module is very good), it follows that 
the vanilla sample complexity of bounding $\prob[g]$ grows like
$1/\prob[g]$, which is excessively big. 

However, if we make some prior assumptions, we can bound $\prob[g]$
using a much smaller number of examples. To motivate the idea, suppose
that $g$ indicates a failure of the sub-module that prevents accidents
with other vehicles. Using a semantic abstraction approach, we will
construct this sub-module by first having a module that detects
vehicles anywhere in the image, and the second module will respond to
detected vehicles only if they are in a dangerous position (e.g.,
immediately in front of us). Denote by $z_1$ the indicator of a
mis-detection of a vehicle somewhere in the image and by $z_2$ the
indicator of a vehicle being in a dangerous position. Then,
\[
\prob[g] = \prob[z_2] \, \prob[z_1 | z_2] ~.
\]
We now introduce the prior assumption that $\prob[z_1 | z_2] \le
\prob[z_1]$. That is, the probability of mis-detection of a car which
is in a dangerous position is at most the probability of mis-detection
of a car in a general position. This is a reasonable assumption
because cars in a dangerous position are close to us, and hence it is
easier to detect them. Furthermore, rare type of cars, on which we
are more likely to err, are more likely being on the side of the road
then immediately in front of us. See \figref{fig:tractor} for an illustration.

\begin{figure}
\begin{center}
\includegraphics[width=0.5\linewidth]{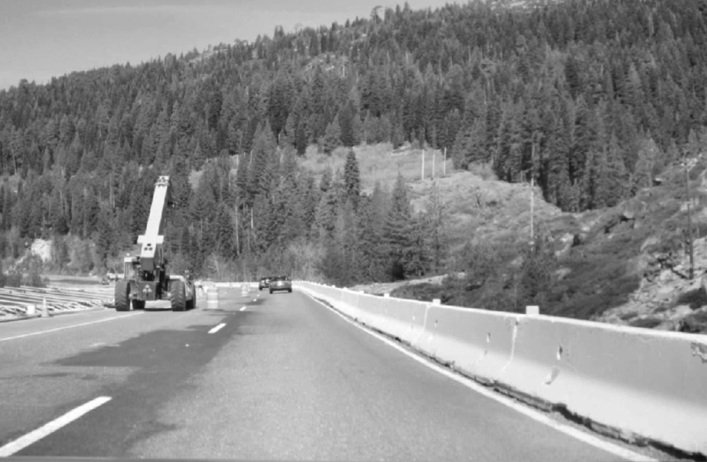}
\end{center}
\caption{The probability of observing a rare type of vehicle
  somewhere in the image is much higher than the probability of
  observing such a vehicle immediately in front of us.} \label{fig:tractor}
\end{figure}

Under the above assumption, we obtain that
\[
\prob[g] ~\le~ \prob[z_2]\, \prob[z_1] ~.
\]
We can continue the same line of thinking by having the vehicle
detection system depends both on a camera-based system and a
radar-based system. To generalize such arguments we rely on the
following definition. 
\begin{definition}[approximate independence]
We say that $z_1$ is $c$-approximately independent of $z_2$ if
\[
\prob[z_1 | z_2] \le c\, \prob[z_1] ~.
\]
\end{definition}
Relying on this definition, we have that if $g = z_1 \land \ldots
\land z_T$, and each $z_t$ is $c$-approximately independent of
$z_1,\ldots,z_{t-1}$, then 
\[
\prob[g] ~\le~ c^T \prod_{t=1}^T \prob[z_t] ~.
\]
The advantage of the above is that if $(c \prob[z_t]) \ll 1$ then
$\prob[g]$ decreases exponentially with $T$.  At the same time, each
$\prob[z_t]$ need not be excessively small to guarantee a very small
bound on $\prob[g]$. For example, if $\prob[z_t]$ is order of
$10^{-6}$, $c = 1.1$, and $T=3$, we end up with $\prob[g] \le 1.34\,\times\,
10^{-18}$. 

To bound each individual $z_t$, we rely
on the following lemma, whose proof follows from Bernstein's inequality (see Lemma
B.10 in \cite{MLbook}):
\begin{lemma}
Consider flipping $m$ times a coin, whose probability to fall on
``head'' is $p$. Let $\hat{p}$ be the fraction of times the coin fell
on ``head''. Then, with probability of at least $1-\delta$ over the
$m$ flips we have
\[
p \le \hat{p} + \sqrt{\frac{2 \hat{p} \log(1/\delta)}{m}} + \frac{4
  \log(1/\delta)}{m} ~.
\]
In particular, if $\hat{p} = 0$ we have that $p \le \frac{4
  \log(1/\delta)}{m}$. 
\end{lemma}

Combining the above lemma with the union bound, we obtain our main result:
\begin{corollary}
Consider a function $g(x) = z_1(x) \land \ldots \land z_T(x)$, and
assume that for every $t$, $z_t$ is $c_t$ independent of
$z_1,\ldots,z_{t-1}$. Let $x_1,\ldots,x_m$ be $m$ random examples
and denote $\hat{p}_t = |\{i : z_t(x_i) = 1\}|/m$. Then, for every
$\delta \in (0,1)$, with probability of at least $1-\delta$ we have
\[
\prob[g] ~\le~ \prod_{t=1}^T \, c_t\, \left( 
\hat{p}_t + \sqrt{\frac{2 \hat{p}_t \log(T/\delta)}{m}} + \frac{4
  \log(T/\delta)}{m}
\right) ~.
\]
\end{corollary}

The interesting fact about the above corollary is that the upper bound
on $\prob[g]$ can be much smaller than $1/m$. That is, it may be the
case that we do not see even a single failure of $g$ on our $m$ training
examples, yet we can guarantee a strong bound on the probability of failure
of $g$. 

So far we have shown how to bound the probability of a failure of a
single sub-module of our full system. We now get back to estimating
the performance of our full system, namely, bounding
$\prob[\ell(x,f(x))]$. Unlike the case of bounding $\prob[g_j(x)]$,
where we made the strong assumption that $g_j$ is an AND of several
events, for the full system we do not impose such a strong assumption on
$\ell(x,f(x))$. We only rely on the very mild assumption that the probability
of our system to fail given that all of its sub-modules work properly
is at most $1/2$.  The following lemma bounds the failure probability
of the entire system in terms of the failure probability of each
sub-module and a residual term. 
\begin{lemma}
Assume that $\prob[\lnot g_1(x) \land \ldots \land \lnot g_k(x)] \ge
0.5$, then
\[
\prob[\ell(x,f(x))] ~\le~ 2\,\sum_{j=1}^k \prob[g_j(x)]  ~+~ \prob[h(x) | \lnot g_1(x) \land \ldots \land \lnot g_k(x)] ~.
\]
\end{lemma}
\begin{proof}
To simplify the notation denote by $h(x) =
\ell(x,f(x))$. Using the law of total probability we can write
\begin{align*}
  \prob[h(x)] &= \prob[g_1(x)]\, \, \prob[h(x) \, | \, g_1(x)] +
  \prob[\lnot g_1(x)]\, \, \prob[h(x) \, | \, \lnot g_1(x)] \\
  &\le \prob[g_1(x)]+
  \prob[h(x) \, | \, \lnot g_1(x)] \\
  &\le \prob[g_1(x)] + \prob[g_2(x) |
  \lnot g_1(x)] + \prob[ h(x) | \land_{i \le
    2} \lnot g_i(x)] \\
&\le \left(\sum_{j=1}^k \prob[g_j(x) | \land_{i < j} \lnot g_i(x)] \right) \,+\, \prob[h(x) | \land_{i \le
    k} \lnot g_i(x)] 
\end{align*}
Using the inequality $\prob[A|B] = \frac{\prob[A \land B]}{\prob[B]} \le \frac{\prob[A]}{\prob[B]}$ and the assumptions in the lemma, the claim follows. 
\end{proof}
To interpret the lemma, recall that we have previously shown how to
upper bound $\prob[g_j(x)]$, and we expect these terms to be
exponentially small. Therefore, their sum is also likely to be rather
small. The last term in the lemma is the probability of a system
failure given that \emph{all} of the sub-modules worked properly. In a
sense, such failures account for cases which are beyond our control
(e.g., in the context of autonomous driving, such a failure can be due
to a sudden tire explosion). We assume that the probability of
such events is close enough to zero.

To summarize, we have shown how utilizing prior knowledge and decomposing
our system into semantically meaningful sub-modules enable us to
upper bound the failure probability of our entire system by a quantity
which is much smaller than $1/m$. This is impossible for
an end-to-end system, where we must see order of $1/\epsilon$ examples
in order to ensure that the probability of failure is below
$\epsilon$.

\bibliographystyle{plain}
\bibliography{bib}

\end{document}